\documentclass{article}
\usepackage{ijcai24}

\usepackage{epstopdf}
\usepackage{mathrsfs}
\usepackage{mathtools}
\usepackage{tabularx}
\usepackage{multirow}
\usepackage{times}
\usepackage{epstopdf}
\usepackage{soul}
\usepackage{url}
\usepackage[hidelinks]{hyperref}
\usepackage[utf8]{inputenc}
\usepackage[small]{caption}
\usepackage{graphicx}
\usepackage{graphics}
\usepackage{bm}
\usepackage{amsmath}
\usepackage{amsthm}
\usepackage{amsfonts}
\usepackage{amssymb}
\usepackage{booktabs}
\usepackage{algorithm}
\usepackage{algpseudocode}
\usepackage[switch]{lineno}
\usepackage{color}

\usepackage{xspace}
\usepackage{dsfont}
\usepackage{paralist}

\usepackage{amsmath,nccmath}
\usepackage{amssymb}
\usepackage{xspace}
\usepackage{xcolor}

\newcommand{\A}{\mathcal{A}} 

\newcommand{\C}{\mathcal{C}} 
\newcommand{\D}{\mathcal{D}}
\newcommand{\E}{\mathcal{E}} 
\newcommand{\F}{\mathcal{F}}
 
\renewcommand{\H}{\mathcal{H}}

\renewcommand{\L}{\mathcal{L}}
\newcommand{\M}{\mathcal{M}} 
\newcommand{\N}{\mathcal{N}}
 
\renewcommand{\P}{\mathcal{P}}
\newcommand{\Q}{\mathcal{Q}} 
\newcommand{\R}{\mathcal{R}}
\renewcommand{\S}{\mathcal{S}} 
\newcommand{\T}{\mathcal{T}}
\newcommand{\U}{\mathcal{U}}

\newcommand{\trace}{\pi}

\newcommand{\LTL}{{\sc ltl}\xspace}
\newcommand{\LTLf}{{\sc ltl}$_f$\xspace}

\newcommand{\LDLf}{{\sc ldl}$_f$\xspace}

\newcommand{\true}{\mathit{true}}

\newcommand{\false}{\mathit{false}}

\newcommand{\Next}{\raisebox{-0.27ex}{\LARGE$\circ$}}
\newcommand{\Wnext}{\raisebox{-0.27ex}{\LARGE$\bullet$}}
\renewcommand{\Until}{\mathop{\U}}

\newcommand{\Release}{\mathop{\R}}

\newcommand{\ls}{{\sf{lst}}}

\newtheorem{definition}{Definition}
\newtheorem{theorem}{Theorem}

\newtheorem{remark}{Remark}

\newtheorem{example}{Example}

\renewcommand{\LTL}{{LTL}\xspace}
\renewcommand{\LTLf}{{LTL}$_f$\xspace}
\renewcommand{\LDLf}{{LDL}$_f$\xspace}

\title{The Trembling-Hand Problem for LTL$_f$ Planning}

\author{
Pian Yu$^1$, Shufang Zhu$^1$, Giuseppe De Giacomo$^1$,
Marta Kwiatkowska$^1$, \And Moshe Vardi$^2$
\affiliations
$^1$Department of Computer Science, University of Oxford, UK \\
$^2$ Department of Computer Science, Rice University, USA \\
\emails
\{pian.yu, shufang.zhu, giuseppe.degiacomo, marta.kwiatkowska\}@cs.ox.ac.uk
\\ vardi@cs.rice.edu
}

\begin{document}

\maketitle

\setlength{\abovedisplayskip}{1pt}
\setlength{\belowdisplayskip}{1pt}

\begin{abstract}
Consider an agent acting to achieve its temporal goal, but with a ``trembling hand". 
In this case, the agent may mistakenly instruct, with a certain~(typically small) probability, actions that are not intended due to faults or imprecision in its action selection mechanism, thereby leading to possible goal failure. We study the trembling-hand problem in the context of reasoning about actions and planning for temporally extended goals expressed in Linear Temporal Logic on finite traces~(\LTLf), where we want to synthesize a strategy (aka plan) that maximizes the probability of satisfying the \LTLf goal in spite of the trembling hand. We consider both deterministic and nondeterministic~(adversarial) domains. We propose solution techniques for both cases by relying respectively on Markov Decision Processes and on Markov Decision Processes with Set-valued Transitions with \LTLf objectives, where the set-valued probabilistic transitions capture both the nondeterminism from the environment and the possible action instruction errors from the agent. We formally show the correctness of our solution techniques and demonstrate their effectiveness experimentally through a proof-of-concept implementation.
\end{abstract}



\section{Introduction}

In this paper, we study the \emph{trembling-hand}~(TH) problem in the context of reasoning about actions and planning for temporally extended goals expressed in Linear Temporal Logic on finite traces \LTLf \cite{DegVa13}.\footnote{All results presented here apply to other linear temporal logics on finite traces, such as \LDLf~\cite{DegVa13} and Pure Past LTL \cite{BonassiAtAlt23}, as long as the set of traces that satisfy a formula can be characterized by a regular language, i.e., by a deterministic finite state automaton.}
In a chess game, a player may have a trembling hand due to, e.g., nervousness, anxiety, or stress, which results in mistaken moves that were not intended. Likewise, an agent acting in an environment could mistakenly instruct a different action, e.g., due to faults, leading to possible goal failure.

The TH problem originates from Game Theory in Economics, see e.g.~\cite{Marchesi021}, referring to the situation where players erroneously select unintended moves with a small quantifiable probability.
This problem highlights the importance of introducing some form of resilience to these errors~\cite{VardiResilience} in the player strategies and has given rise to the well-known notion of Trembling Hand Perfect Equilibrium in Economics~\cite{Bielefeld1988}. 

Here, we study this problem in the context of reasoning about actions~\cite{Reiter2001} and planning~\cite{GeffnerBonet2013}.
Specifically, we consider an agent acting in a domain. At each state of the domain, when the agent instructs an action, with a certain probability, it can mistakenly instruct a different action. Notice that this uncertainty is on the agent decision-making capabilities, not on how the environment executes the instructed actions. To stress this point, we consider two settings: deterministic domains, where the environment has no choices in responding to agent actions, as in classical planning \cite{GeffnerBonet2013}; and nondeterministic domains, where the environment can adversarially respond to agent actions, as in Fully Observable Nondeterministic Domains (FOND), when considering strong plans \cite{CimattiPRT03,FOND,GeffnerBonet2013,DeGiacomoR18}. In both settings, we want to synthesize a strategy (aka plan) that guarantees to maximize the probability of fulfilling a temporally extended goal expressed in \LTLf, in spite of the adversarial response of the environment in the case of nondeterministic domains.

We devise solution techniques to solve the problem in the two settings by relying respectively on Markov Decision Processes (MDPs)~\cite{puterman2014markov} in the case of deterministic domains and on Markov Decision Processes with Set-valued Transitions~(MDPSTs)~\cite{trevizan2007planning,trevizan2008mixed} in the case of nondeterministic domains.
MDPs specify concrete probability values for each transition. MDPs with imprecise probabilities~(MDPIPs) \cite{white1994markov,satia1973markovian,givan2000bounded} and Uncertain MDPs~(UMDPs)~\cite{nilim2004robust,buffet2005robust,hahn2019interval} have been proposed for scenarios where the probability values are uncertain. MDPSTs constitute a restricted subclass of MDPIPs, which combine probabilistically quantifiable uncertainty with unquantifiable uncertainty (nondeterminism) in a unified framework. 
They are attractive because they admit a simplified Bellman equation compared to MDPIPs, UMDPs~\cite{trevizan2007planning}, and thus stochastic games.


In both settings, we consider \LTLf objectives instead of standard reachability. Note that MDPs with \LTLf objectives have been studied in \cite{BrafmanGP18,WellsKLKV21}. Instead, MDPSTs with \LTLf objectives are studied for the first time in this paper. We lift the definition of satisfying \LTLf objectives from MDPs to MDPSTs by defining the notion of robust strategy, and then an efficient value iteration algorithm is proposed for synthesizing an optimal robust strategy. We evaluate the effectiveness of the proposed solution techniques on a human-robot co-assembly problem, where the robot operates with a trembling hand, and demonstrate promising scalability. 

\subsection*{Related Work} 
Interestingly, the trembling-hand problem has never been specifically studied in reasoning about actions and planning, though some related work exists. For example, the classical work on reasoning with noisy sensors and effectors in Situation Calculus \cite{bacchus1999reasoning} is indeed related. There, sensors and effectors are considered to be noisy, so they introduce a stochastic element to be taken into account in reasoning. In particular, uncertainty on the effectors may be considered similarly to our trembling hand.  There is, however, an important distinction between the two: in our case, the uncertainty is on which action is actually \emph{instructed by the agent}; instead, in their case, the uncertainty is on how the action (perfectly instructed) is actually \emph{executed by the environment}. In other words, in our case, the uncertainty is on the agent, while in theirs, the uncertainty is on the environment, which has been more extensively studied since it corresponds to uncertainty in modeling the reaction of the world to agent moves.  In fact, there is a growing interest in forms of synthesis/planning that are resilient to errors in modeling the environment. For example, in~\cite{ZhuDeg22}, one could compute a maximally permissive strategy for the agent, allowing it to switch from one strategy to another while in execution, in case exceptional environment changes occur such that the predetermined strategy fails. In~\cite{AminofGLMR21,CiolekDPS20}, multiple models of the environment are considered to handle exceptions during planning, aiming to mitigate the intrinsic risk in a single environment model. In~\cite{AinetoGGRSS23}, a $k$-resilient strategy allows the agent to fail $k$ times maximally at a repeatedly occurring state. 

Nevertheless, in all these works, the focus is on errors/exceptions wrt expected environment behaviors. The trembling hand, on the other hand, is about errors in the agent behavior, and, as such, has not been much studied yet. Only two papers~\cite{WellsLKV20,WellsKLKV21} touched on this aspect, accounting for possible errors in robot decisions. Nevertheless, the environment considered is either deterministic or has probabilistic uncertainty. In contrast, our environment admits adversarial behaviors.



\section{Preliminaries}

We study the \emph{trembling hand}~(TH) problem in the context of planning for temporally extended goals expressed in \LTLf. We now briefly introduce the logic \LTLf,  deterministic and nondeterministic planning domains, and the notion of strategy~(aka plan) in a domain achieving an \LTLf goal.

\noindent\textbf{LTL$_f$.} 
\textit{Linear Temporal Logic on finite traces}~(\LTLf) is a specification language expressing temporal properties on finite, nonempty traces. In particular, \LTLf
shares syntax with \LTL, which is instead interpreted over infinite traces~\cite{Pnu77}. Given a set of atomic propositions $Prop$, \LTLf formulas are generated as follows: 
{\centerline{$\varphi ::= a \mid \varphi \wedge \varphi \mid \neg \varphi \mid  
 \Next \varphi \mid \varphi \mathop{\U} \varphi,$}}
where $a \in Prop$ is an \textit{atom}, $\Next$~(\emph{Next}) and $\mathop{\U}$~ (\emph{Until}) are temporal operators. 
We make use of standard Boolean abbreviations, e.g., $\vee$~(or) and $\rightarrow$~(implies), $\true$ and $\false$. In addition, we define the following abbreviations: \emph{Weak Next} $\Wnext \varphi \equiv \neg \Next \neg \varphi$, \emph{Eventually} $\Diamond \varphi \equiv \true \Until \varphi$ and \emph{Always} $\Box \varphi \equiv \false \Release \varphi$, where $\Release$ is for \emph{Release}.

A \textit{trace} $\trace = \trace_0\trace_1\ldots$ is a sequence of propositional interpretations~(sets), where for every $i \geq 0$, $\trace_i \in 2^{Prop}$ is the $i$-th interpretation of $\trace$. Intuitively, $\trace_i$ is interpreted as the set of propositions that are $true$ at instant $i$. We denote the last instant~(i.e., index) in a trace $\trace$ by $\ls(\trace)$.  
By $\trace^k = \trace_0 \cdots \trace_k$ we denote the \emph{prefix} of $\trace$ up to the $k$-th iteration, and $\trace^k = \epsilon$ denotes an empty trace if $k < 0$. 
We denote $\trace$ \emph{satisfies} $\varphi$ by $\trace \models \varphi$. The detailed semantics of \LTLf can be found in~\cite{DegVa13}.
It is also shown there that, for every \LTLf formula $\varphi$, one can construct a Deterministic Finite Automaton~(DFA) $Aut_\varphi = (2^{Prop}, \Q, q_0, \delta, acc)$, where $2^{Prop}$ is a finite alphabet, $\Q$ is a finite set of states, $q_0 \in \Q$ is the initial state, $\delta : \Q \times 2^{Prop} \rightarrow \Q$ is the transition function, and $acc$ is the set of accepting states, such that for every trace $\pi$ we have $\pi\models \varphi$ iff $\pi$ is accepted by $Aut_\varphi$.


%

\noindent\textbf{Deterministic Domain.} A deterministic domain is a tuple $\D = (\S, s_0, A, F_d, \L)$, where $\S$ is a finite set of states, $s_0 \in \S$ is an initial state, $A$ is a finite set of \emph{actions}, $F_d: \S \times A \mapsto \S$ is the deterministic transition function, 
where $s' = F_d(s,a)$ is the successor state after performing an \emph{applicable} action $a$ at $s$. We use $A(s)\subseteq A$ to denote the set of \emph{applicable} actions at $s$. $\L: \S \rightarrow 2^{Prop}$ is the labeling function, where $Prop$ is a finite set of propositions. Note that, compared to typical formulations of domains in planning~\cite{GeffnerBonet2013}, 
we are assuming that more than one state can have the same evaluation of the propositions~(fluents).


\noindent\textbf{Nondeterministic Domain.} 
A nondeterministic domain is a tuple $\N = (\S, s_0, A, F_n, \L)$, where $\S, s_0, A$, and $\L$ are defined as in deterministic domains, and $F_n: \S \times A \mapsto 2^{\S}$ is now a nondeterministic transition function such that $F_n(s,a)$ denotes the non-empty set of possible successor states that follow an applicable action $a \in A(s)$ in $s$.
It is worth noting that domains that are typically compactly represented, e.g., in Planning Domain Description Language~(PDDL)~\cite{PDDLbook}, can 
be encoded using a number of bits that is logarithmic in the number of states. 

\noindent\textbf{\LTLf Planning.} A \emph{path} of $\D$~(resp.~$\N$) is a finite or infinite sequence of alternating states and actions $\rho = s_0 a_0 s_1 a_1\cdots$, ending with a state if finite, where $s_0$ is the initial state, and $F_d(s_{i}, a_{i}) = s_{i+1}$~(resp.~$s_{i+1} \in F_n(s_{i}, a_{i})$) for all $i$ with $0 \leq i < |\rho|$, and $|\rho| \in \mathds{N} \cup \{\infty\}$. We denote by $\rho^k = s_0 a_0 s_1 a_1 \cdots s_k$ the finite prefix of $\rho$ up to the $k$-th alternation.
The sequence $\trace(\rho, \D) = \L(s_0)\L(s_1)\cdots$~(resp.~$\trace(\rho, \N) = \L(s_0)\L(s_1)\cdots$) over $Prop$ is called the \emph{trace} induced by $\rho$ over $\D$~(resp.~$\N$). ${\rm FPaths}$ denotes the set of all finite paths.
An \emph{agent strategy}~(or plan) is a function $\sigma_p: {\rm FPaths}\rightarrow A$ mapping a finite path on $\D$~(resp.~$\N$) to agent actions that are applicable at the last state of the finite path.
For nondeterministic domains, we assume that the nondeterminism is resolved by the environment, which acts according to an (unknown) strategy as an (adversarial) antagonist of the agent. Environment strategies are functions of the form  $\gamma_p: {\rm FPaths} \times A \rightarrow \S$, which need to comply with the domain, in the sense that given a finite path $\rho \in {\rm FPaths}$ and an action $a \in A(\ls(\rho))$ it must be the case that $\gamma_p(\rho)\in F_n(\ls(\rho),a)$.  Note that, in the case of deterministic domains, this constraint forces the environment to have only one strategy. Given an agent strategy $\sigma_p$ and an environment strategy $\gamma_p$, there is a unique path $\rho(\sigma_p, \gamma_p) = s_0 a_0 s_1 a_1\cdots$ generated by $\sigma_p$ and $\gamma_p$, where $s_0$ is the initial state and for every $i \geq 0$ it holds that $a_i = \sigma_p(\rho^{i})$ and $s_{i+1} = \gamma_p(\rho^i, a_i)$. Sometimes, for simplicity, we write $\rho$ instead of $\rho(\sigma_p, \gamma_p)$, when it is clear in the context.

An agent strategy $\sigma_p$ enforces an \LTLf $\varphi$ in a domain $\D$~(resp.~$\N$) if for every environment strategy $\gamma_p$, the infinite path $\rho(\sigma_p, \gamma_p)$ contains a finite prefix $\rho^i$ such that the finite trace $\pi(\rho^i, \D)\models\varphi$~(resp.~$\pi(\rho^i, \N)\models\varphi$). \LTLf planning concerns computing such an agent strategy $\sigma_p$, if one exists.

\section{TH in Deterministic Domains}

We begin investigating the \emph{trembling-hand} problem for \LTLf planning by focusing on the case of deterministic domains, where the environment has only one strategy, i.e., following the transitions of the domain. Hence, the only uncertainty is the stochastic one, arising from the ``trembling hand". In the next section, we consider the case of nondeterministic domains, where the environment employs adversarial strategies.

\subsection{Problem formulation}
The ``trembling hand" refers to the agent intending to instruct a certain action but, by mistake, instructing a different action with a~(small) quantified probability. This probability only depends on the domain state where the action is instructed~(and then performed).
Notice that, here, the environment is fully deterministic, and hence, once the action is instructed, it will be executed (without any error) in a deterministic way.
The problem that we want to address is to maximize the probability of achieving a given \LTLf objective in a deterministic domain in spite of the~(state dependent) action-instruction errors due to the ``trembling hand''.


We formalize the action-instruction errors as follows. 
Let $\D = (\S, s_0, A, F_d, \L)$ be a deterministic domain, $s \in \S$ a domain state, and $a\in A(s)$ an applicable action at $s$.
We denote by $err(s, a)\in {\rm Dist}(A(s))$ the probability distribution of instructing an action $a'$ instead of $a$ in state $s$. For instance, suppose the set of applicable actions at $s$ is such that $A(s) = \{a, a', a''\}$ and $err(s, a) = [0.9, 0.04, 0.06]$. This means that,  when the agent intends to instruct action $a$ at state $s$, then with probability $0.04$ and $0.06$, it may instruct actions $a'$ and $a''$, respectively. Let $\E = \{err(s, a): (s, a) \in \S \times A\}$ be the set of the (state-dependent) action-instruction errors caused by the ``trembling hand''. 
The set of actions that could be instructed when the agent intends to instruct action $a$ at state $s$ is the \emph{support set} of $err(s, a)$, denoted by $supp(s, a) \subseteq A(s)$. In this example, $supp(s,a) = \{a, a', a''\}$.


Recall that, without a trembling hand, executing an agent strategy $\sigma_p$ in a deterministic domain $\D$ results in a unique path. Yet, in the presence of a ``trembling hand", i.e., action-instruction errors $\E$, we get perturbed paths, where the actually instructed action may differ from the intended one with a probability following $\E$.


\begin{definition}[Perturbed path in $\D$]\label{def:pert_path}
    Let $\D$ be a deterministic domain, $\sigma_p$ an agent strategy, and $\E$ a set of action-instruction errors. A \emph{perturbed path} in $\D$ wrt $\E$ is a sequence of triples $\rho(\sigma_p,\E) = (s_0, a_0, a'_0)(s_1, a_1, a'_1)\cdots$,  
    where for every $i \geq 0$, $a_i= \sigma_p(\rho^{i})$~(the intended action)$, a'_i \in supp(s_i, a_i)$~(the actually instructed action), and $s_{i+1} = F_d(s_i, a'_i)$. The set of all perturbed paths in $\D$ wrt $\sigma_p$ and $\E$ is denoted by $\Phi^{\sigma_p, \E}$.
\end{definition}

Given a deterministic domain $\D$, an \LTLf formula $\varphi$, and a set of action-instruction errors $\E$, the probability of $\sigma_p$ enforcing $\varphi$ in $\D$ wrt $\E$
is defined as ${\rm Pr}_{\D}^{\sigma_p, \E}(\varphi) := {\rm Pr}_{\D}(\{\rho \in \Phi^{\sigma_p, \E} \mid \pi(\rho^k, \D) \models \varphi \text{ for some } k \geq 0\}).$


\begin{definition}[TH problem for \LTLf planning in $\D$]\label{DD}
The problem is a tuple $\P^d = (\D, \varphi, \E)$, where $\D$ is a deterministic domain, $\varphi$ is an \LTLf formula, and $\E$ is a set of action-instruction errors. Solving $\P^d$ consists in synthesizing an agent strategy $\sigma_p^*$ that maximizes the probability of enforcing $\varphi$ in $\D$ wrt $\E$, i.e., an \emph{optimal strategy} for $\P^d$, that is:
$\sigma_p^* = \arg\max_{\sigma_p} {\rm Pr}_{\D}^{\sigma_p, \E}(\varphi)$.
\end{definition}

In the following, an example is given to demonstrate a perturbed transition in a deterministic domain $\D$.

\begin{example}\label{example1}
    Consider a robot assembly problem, where the robot aims to assemble an arch using $N$ blocks. In Figure~\ref{Fig:arch}, the goal configurations for $N~(2 \leq N \leq 6)$ blocks are depicted (which are used later in Section 5). Let $OBJ = \{Obj_i \mid i\in \{1, \cdots, N\}\}$ be the set of blocks and $LOC = \{L_j \mid j\in \{0, \cdots, M\}\}$ be the set of locations, where $L_0$ represents the storage. Initially, all the blocks are stored in the storage.

    During assembly, the robot can perform \emph{move} actions to relocate blocks. The  set of robot actions is $A = \{(Obj_i, L_j): Obj_i\in OBJ, L_j\in LOC\} \cup \{\mbox{do-nothing}\}$, where $(Obj_i, L_j)$ means move block $i$ to location $j$. Due to the ``trembling hand" (which may be caused by drifting), the robot’s action is subject to uncertainty. For instance, if the robot intends to move block $i$ to location $j$, there exists a probability that it may mistakenly move a different block $i' \neq i$ or inaccurately place it in location $j' \neq j$. In addition, the probability of errors varies for different actions. For instance, if the robot chooses ${\textsc{Do-nothing}}$, one can safely assume that the probability of error is 0. However, if the robot chooses to \emph{move}, we assume a positive probability of error, leading to perturbed transitions resulting in perturbed paths.

    Figure \ref{Fig:mdp} shows a perturbed transition example in this case. The dashed arrow shows the intended action and the solid arrows represent the set of actions that may be actually instructed with their respective probabilities.
\end{example}
 \begin{figure}[H]
\centering
\includegraphics[width=0.25\textwidth]{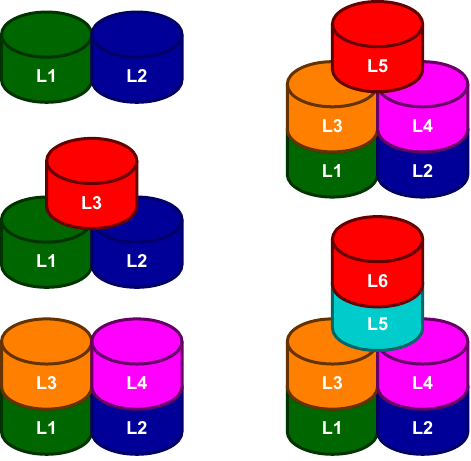}
	\caption{An arch. Left: 2, 3, and 4 blocks. Right: 5 and 6 blocks. }
	\label{Fig:arch}
\end{figure}

\begin{figure}[H]
\centering
\includegraphics[width=0.32\textwidth]{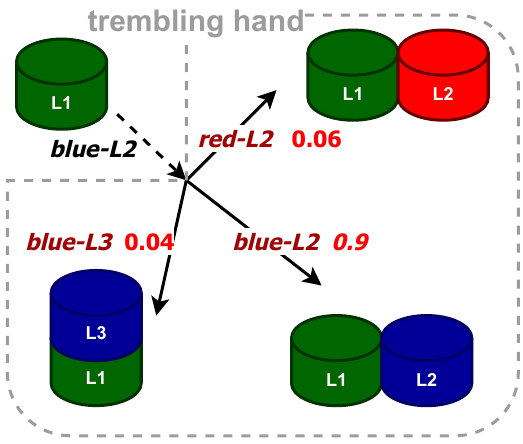}
\caption{A perturbed transition in a det. domain.}
\label{Fig:mdp}
\end{figure}

\subsection{MDPs with \LTLf objectives}
Solving $\P^d = (\D, \varphi, \E)$ requires a modeling technique that is able to capture the action-instruction errors $\E$ in a deterministic domain. To this end, we first review \emph{Markov Decision Processes~(MDPs)}, which are used in our solution technique.

\emph{MDPs} allow action choice in each state, and each state action transition is a probability distribution on successor states, which provides a natural way of capturing the probability of the agent mistakenly taking a different action in a deterministic domain. 
Following \cite{puterman2014markov}, 
an MDP is a tuple $\M = (\S, s_0, A, \T,  \L)$, where $\S$ is a finite set of states, $s_0\in \S$ is the initial state, $A$ is a finite set of actions, $\T: \S \times A \times \S \rightarrow [0, 1]$ is the probabilistic transition function, and $\L: \S \to 2^{Prop}$ is the proposition labelling function. Analogously to planning domains, for each $s\in \S$, the set of actions \emph{applicable} at state $s$ is denoted by $A(s)$.  

Denote by $\xi=s_0 a_0 s_1 a_1\cdots$ a path and by ${\rm FPaths}$ the set of all finite paths of $\M$.
In this work, we focus on \emph{deterministic} agent strategies for MDPs, instead of \emph{randomized}. A strategy $\sigma_m$ of $\M$ is a function $\sigma_m: {\rm FPaths} \to A$ such that, for each $\xi\in {\rm FPaths}$, $\sigma_m(\xi)\in A({\ls}(\xi))$, where ${\ls}(\xi)$ is the last state of $\xi$. We denote by $\Xi^{\sigma_m}$ the set of all probably infinite paths of $\M$ generated by $\sigma_m$.

Given an MDP $\M = (\S, s_0, A, \T,  \L)$ and a set of goal states $G \subseteq \S$, the probability of an agent strategy $\sigma_m$ enforcing $G$ in $\M$ is defined as ${\rm Pr}_{\M}^{\sigma_{m}}(G) := {\rm Pr}_{\M}(\{\xi\in \Xi^{\sigma_m} \mid \ls(\xi^k) \in G \text{ for some } k \geq 0\})$.
Computing an optimal strategy $\sigma^*_m$ that maximizes the probability of enforcing $G$ is the \emph{reachability problem} over $\M$.  Analogously, an agent strategy can also enforce a temporal objective on an MDP. Given an MDP $\M$ and an \LTLf formula $\varphi$, the probability of an agent strategy $\sigma_m$ enforcing $\varphi$ in $\M$ is defined as ${\rm Pr}_{\M}^{\sigma_{m}}(\varphi) := {\rm Pr}_{\M}(\{\xi\in \Xi^{\sigma_m} \mid \pi(\xi^k, \M) \models \varphi \text{ for some } k \geq 0\})$.
The problem of MDP with \LTLf objective is to compute an optimal strategy $\sigma^*_m$, which maximizes the probability of enforcing $\varphi$ in $\M$~\cite{baier2008principles}.

\subsection{Solution technique}\label{sol:THDet}


We now present our solution technique to synthesize an optimal agent strategy $\sigma_p^*$ for $\P^d = (\D, \varphi, \E)$, which aims to maximize the probability of enforcing $\varphi$ in $\D$ in spite of $\E$. The key idea is to reduce $\P^d$ to an MDP with an \LTLf objective. Intuitively, the MDP provides a probabilistic abstraction of instructing mistaken actions in the domain due to the trembling hand. Hence, the MDP has the same states as the domain, retains its original action choices, and incorporates the probability of action-instruction errors in its transitions. 

\noindent\textbf{Probabilistic abstraction.} 
We define an MDP $\mathcal{M} = (\S, s_0, A, \T, \L)$ from $\P^d = (\D, \varphi, \E)$ as follows. $\S, s_0, A$, and $\L$ are the same as in $\D$. In order to construct the probabilistic transition function $\T: \S \times \A \times \S \to [0, 1]$, note that $\T(s, a, s')$ gives the probability of transitioning from state $s$ to $s'$ on action $a$. Assuming no ``trembling hand'' errors,  we have that $\T(s, a, s') = 1$ if $s' = F_d(s, a)$ and $\T(s, a, s') = 0$ otherwise. Due to ``trembling-hand'' errors, however, unintended actions (e.g., $a'$) may be instructed with some probability (given by $err(s, a)(a')$), thus resulting in a different successor state $s'' = F_d(s, a')$. Following this observation, we construct $\mathcal{T}$ as follows:
    \begin{equation*}
    \mathcal{T}(s, a, s') = \begin{cases}
        err(s, a)(a'), & \text{if $a' \in A(s), s'=F_d(s, a'),$}\\
        0, & \text{otherwise.}       
    \end{cases}   
\end{equation*}

The TH problem for \LTLf planning in deterministic domains is now reduced to an MDP with an \LTLf objective. 

\begin{theorem}\label{thm:toMDP}
Let $\P^d = (\D, \varphi, \E)$ be a TH problem defined in Def. \ref{DD}, and $\M$ the constructed MDP described above. An optimal strategy for $\M$ with $\varphi$ is an optimal strategy for $\P^d$ and vice versa, that is: 
  \begin{equation*}
        \sigma_p^* =  \arg\max_{\sigma_m}\{{\rm Pr}_{\M}^{\sigma_m}(\varphi)\},
\end{equation*}
where $\sigma_p^*$~(optimal strategy for $\P^d$) is given in Def. \ref{DD}.
\end{theorem}

\begin{proof}
First, we observe that $\sigma_p = \sigma_m$, by construction. Second, one can derive, according to Def.~1 that 
$\Phi^{\sigma_p, \E} = \Xi^{\sigma_{m}}$ if $\sigma_p = \sigma_m$. 
Finally, given a state $s$ and an intended action $a = \sigma_p(s)$ in $\P^d$, one has that,  with probability $err(s, \sigma_p(s))(a')$, the agent may actually instruct a different action $a'$ and then transit to $s' = F_d(s, a')$. Denote by ${\rm Pr}_{\P^d}(s' | s, \sigma_p(s))$ the probability of transiting to state $s'$ from the state $s$ with action $\sigma_p(s)$. We obtain that if $\sigma_p(s) = \sigma_m(s)$, then ${\rm Pr}_{\P^d}(s' | s, \sigma_p(s)) = err(s, \sigma_p(s))(a') = \T(s, \sigma_{m}(s), s').$
Hence, the conclusion follows.
\end{proof}

Thm.~\ref{thm:toMDP} allows us to utilize existing algorithms for MDPs with \LTLf objectives to solve the TH problem for \LTLf planning in deterministic domains. The common approach 
is by reduction to the reachability problem of an MDP~\cite{baier2008principles,WellsKLKV21}. 
More specifically, given an MDP $\M$ with \LTLf objective $\varphi$, we first construct the corresponding DFA $Aut_\varphi$ of the \LTLf formula $\varphi$, then construct the product MDP $\mathcal{M}^{\times}$ of $\M$ and $Aut_\varphi$. In this case, computing an optimal strategy for $\M$ with $\varphi$ reduces to the reachability problem over $\mathcal{M}^{\times}$, where the goal states $G$ are those in $\M^\times$ that consist of the accepting states of $Aut_\varphi$. The reachability problem of $\M^\times$ wrt $G$ can be solved via value iteration, strategy iteration, or linear programming~\cite{altman1998constrained,WellsKLKV21}. Due to the cross product, every finite path $\rho \in {\rm FPaths}_\D$ such that $\rho = s_0a_0s_1a_1 \cdots s_{k+1}$, where ${\rm FPaths}_\D$ denotes the finite paths on $\D$, corresponds to a finite path $\rho^\times$ on $\M^\times$, where $\rho^\times = (s_0,q_0)a_0(s_1,q_1)a_1 \cdots(s_k,q_k)a_k(s_{k+1},q_{k+1})$. 
Therefore, every strategy $\sigma^*_m$ for $\mathcal{M}^{\times}$ wrt $G$ induces a strategy for $\P^d = (\D, \varphi, \E)$ as follows: $\sigma^*_p(\rho) = \sigma^*_m(\rho^\times)$ for $\rho \in {\rm FPaths}_\D$. Together with Thm.~\ref{thm:toMDP}, the following theorem is an immediate result.

\begin{theorem}\label{thm:toMDPfinal}
Let $\P^d = (\D, \varphi, \E)$ be a TH problem defined in Def. \ref{DD}, $\M^\times = \M \times Aut_{\varphi}$ the constructed product MDP, and $G$ the set of goal states. We have that the computed optimal strategy $\sigma^*_m$ of $\M^\times$ with reachability goal $G$ induces an optimal strategy $\sigma^*_p$ for $\P^d$.
\end{theorem}
\section{TH in Nondeterministic Domains}

We now turn to the case in which the domain is nondeterministic, i.e., the environment can choose its strategy adversarially. Hence, in planning, we have to overcome two forms of uncertainty: the stochastic uncertainty from the trembling hand and the adversarial uncertainty from the environment. 

\subsection{Problem formulation}

Consider the case where the agent acts in a nondeterministic (adversarial) domain. We assume the nondeterminism is unquantifiable, so it is adversarial without a probabilistic behavior. In this setting, we want to synthesize an agent strategy that maximizes the probability of achieving its goal in spite of the adversarial behavior of the environment and the~(small) probability of instructing wrong actions at every step. In other words, the agent seeks a maxi-min strategy, that is, a strategy that maximizes the minimal probability across all possible environment strategies. We find this case to be particularly interesting, since it combines the probabilistic aspects of the previous case's action-instruction errors and the environment's adversarial nondeterminism.

We first define perturbed paths of an agent acting in a nondeterministic domain $\N$ with action-instruction errors $\E$, i.e., the trembling hand, and adversarial environment behaviors.

\begin{definition}[Perturbed paths in $\N$]\label{def:pert_path_N}
    Let $\N$ be a nondeterministic domain, $\sigma_p$ an agent strategy, $\gamma_p$ an environment strategy, and $\E$ a set of action-instruction errors. A \emph{perturbed path} in $\N$ wrt $\E$ is a sequence of triples $\rho(\sigma_p, \gamma_p, \E) = (s_0, a_0, a'_0)(s_1, a_1, a'_1)\cdots$. We denote by $\rho'$ the projection of $\rho$ by considering only the states $s_i$ and actually instructed actions $a'_i$. It holds that, for every $i \geq 0$, $a_i\in \sigma_p({\rho'}^{i})$~(the intended action)$, a'_i \in supp(s_i, a_i)$~(the actually instructed action), and $s_{i+1} = \gamma_p({\rho'}^{i}, a'_i)$. The set of all perturbed paths in $\N$ wrt $\sigma_p, \gamma_p, \E$ is denoted by $\Phi^{\sigma_p, \gamma_p, \E}$.
\end{definition}


The probability of an agent strategy $\sigma_p$ enforcing $\varphi$ considering environment strategy $\gamma_p$ and the action-instruction error $\E$ is denoted by 
${\rm Pr}_{\N}^{\sigma_p, \gamma_p, \E}(\varphi)$.

\begin{definition}[TH problem for \LTLf planning in $\N$]\label{ND}
The problem is a tuple $\P^n = (\N, \varphi, \E)$, where $\N$ is a nondeterministic domain, $\varphi$ is an \LTLf formula, and $\E$ is a set of action-instruction errors. Solving $\P^n$ consists in synthesizing an agent strategy $\sigma^*_p$ that maximizes the probability of enforcing $\varphi$ in $\N$ with $\E$ in spite of adversarial strategies of the environment, i.e., an \emph{optimal strategy} for $\P^n$, that is: 
\begin{equation*}
\sigma^*_p = \arg\max_{\sigma_p}\min_{\gamma_p} {\rm Pr}_{\N}^{\sigma_p, \gamma_p, \E}(\varphi).
\end{equation*}
\end{definition}

\begin{example}\label{example2}
Consider the human-robot co-assembly problem adapted from \cite{he2019efficient}, where the robot aims to assemble an arch along with a human in a shared workspace. During assembly, the robot can perform actions to \emph{relocate} blocks, and the human may perform moves to intervene. The robot has a ``trembling hand" as described in Example \ref{example1}. Compounding this issue, the robot must ensure that the arch is successfully built, considering the human involvement. 

\begin{figure}[H]
    \centering
    \includegraphics[width=0.5\textwidth]{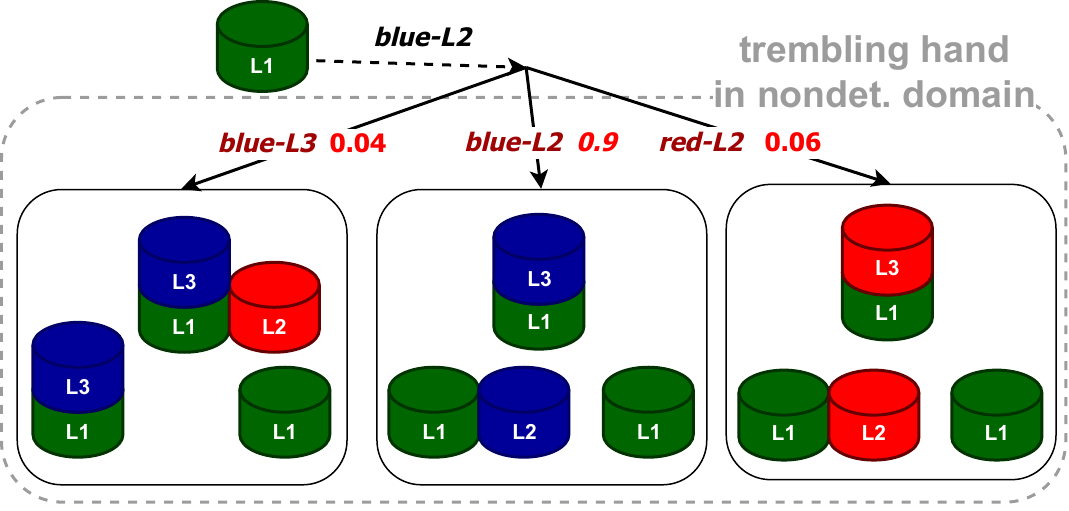}
    \caption{A perturbed transition in a nondet. domain. }
    \label{Fig:mdpst}
\end{figure}

At each step, the robot instructs an action to relocate a block, though with a trembling hand, and then the human can react by moving blocks among locations. Human movements are not controllable, thus introducing nondeterminism into the consequences of instructed robot actions. Figure~\ref{Fig:mdpst} shows a perturbed transition example of a robot working in such a nondeterministic domain with a trembling hand. Note that each intended action~(dashed arrow) corresponds to a set of actions~(solid arrow) that may be actually instructed, together with their respective probabilities. Furthermore, each actually instructed action leads to a set of possible changes due to uncontrollable human intervention.


\end{example}

\subsection{MDPSTs with \LTLf objectives}
We now introduce \emph{Markov Decision Processes with Set-valued Transitions (MDPSTs)}, which combine probabilistically quantifiable uncertainty with unquantifiable uncertainty~(nondeterminism), hence providing a natural probabilistic abstraction of 
an agent acting with a trembling hand in a nondeterministic domain. Note that MDPSTs are favorable due to their simplified Bellman equation compared to MDPIPs, 
UMDPs, and thus stochastic games. This distinction will be further clarified in Section 4.3.

An \emph{MDPST} is a tuple $\M_N = (\S, s_0, A, \F, \T_N, \L)$, where $\S, s_0, A$, and $\L$ are defined as for MDPs, and
\begin{compactitem}
    \item 
    $\F: \S \times A \Rightarrow 2^\S$ is the set-valued nondeterministic transition function;
    \item 
    $\T_N: \S \times A \times 2^\S \to [0, 1]$ is the transition probability~(or mass assignment) function, i.e., given a set $\Theta\in \mathcal{F}(s, a)$, where $\Theta \subseteq \S$, $\mathcal{T}_N(s, a, \Theta)={\rm Pr}(\Theta|s, a)$.
\end{compactitem}
Note that, in MDPSTs, the transition function $\mathcal{F}(s, a)$ returns a set of state sets, i.e., $\F(s,a) \subseteq 2^\S$, and the transition probability function $\mathcal{T}$ expresses the probability of transitioning to such sets via a given action.


It is worth noting that the problem of MDPSTs with \LTLf objectives is studied for the first time in this paper.
To define the problem, one needs to lift the definition of enforcing \LTLf objectives from MDPs to MDPSTs by defining the notion of robustness for strategies to incorporate the unquantifiable uncertainty. In an MDPST, due to the unquantifiable uncertainty, the 
distribution on the set of reachable successor states of a state-action pair $(s,a)$, where $a \in A(s)$, is not uniquely determined by $(s,a)$, in contrast to MDPs. We denote by ${\rm Post}(s, a)$ the set of reachable states of $(s,a)$, despite both quantifiable and unquantifiable uncertainties. Formally, ${\rm Post}(s, a) = \{s' \mid \exists \Theta\in \mathcal{F}(s, a) \text{ s.t. } \mathcal{T}_N(s, a, \Theta)>0$ and $  s'\in \Theta \}$. A \emph{feasible distribution} of an MDPST guarantees that, given a state action pair $(s,a)$, where $a \in A(s)$, $\emph{(i)}$ only one state is chosen within $\Theta$ for each $\Theta\in \mathcal{F}(s, a)$; $\emph{(ii)}$ the sum of probabilities of selecting a state from ${\rm Post}(s,a)$ equals 1; and $\emph{(iii)}$  the probability of selecting a state out of ${\rm Post}(s,a)$ is 0. 
In the following definition, 
$\iota^{\Theta}_{s'}$ indicates whether $s'$ is \emph{in} ${\Theta}$. Hence $\iota^{\Theta}_{s'}=1$ if $s'\in \Theta$ and $\iota^{\Theta}_{s'}=0$ otherwise. Furthermore, $\alpha_{s'}^{\Theta}$ indicates whether $s'$ is \emph{selected} from ${\Theta}$. Hence $\alpha_{s'}^{\Theta} = 1$ if $s'$ is selected from $\Theta$, and $\alpha_{s'}^{\Theta} = 0$ otherwise.

\begin{definition}[Feasible distribution in MDPSTs]\label{Def:feasibledistribution}
    Let $\mathcal{M}_N=(\S, s_0, A, \mathcal{F}, \mathcal{T}_N, \L)$ be an MDPST, $(s, a)$ a state-action pair, where $a \in A(s)$. $\mathfrak{h}_s^a\in {\rm Dist}(\S)$ is a \emph{feasible distribution} of $(s, a)$, denoted by $s \xrightarrow[]{a} \mathfrak{h}_s^a$, if 
    \begin{compactitem}
        \item[(i)] 
        $\sum_{s'\in \Theta} \alpha_{s'}^{\Theta} =1, \mbox{ for } \Theta\in \mathcal{F}(s, a)$;
        \item[(ii)] $\mathfrak{h}_x^a(s'){=}\sum\limits_{\Theta\in \mathcal{F}(s, a)}\iota^{\Theta}_{s'}\alpha_{s'}^{\Theta} \mathcal{T}_N(s, a, \Theta), 
        \mbox{ for } s'{\in} {\rm Post}(s, a)$; 
        \item[(iii)] $\mathfrak{h}_s^a(s')=0, \mbox{ for } s'\in \S\setminus {\rm Post}(s, a)$.
    \end{compactitem}  
\end{definition}

Following Def.~\ref{Def:feasibledistribution}, it is evident that, in MDPSTs, a feasible distribution is not uniquely determined for a given state-action pair, in contrast to MDPs.
%
It highly depends on $\alpha_{s'}^{\Theta}$~(see item $(ii)$), the unquantifiable uncertainty of whether $s'$ is selected from $\Theta$. We now introduce \emph{nature} for MDPSTs to characterize this unquantifiable uncertainty, motivated by the definition of \emph{nature} in robust MDPs~\cite{nilim2004robust}. One can intuitively perceive nature as the environment in the context of nondeterministic domains, playing a role in resolving nondeterminism. We denote by $\H_s^a$ the set of feasible distributions of state action pair $(s,a)$. Analogously to MDPs, ${\rm FPaths}$ denotes the set of finite paths of an MDPST.

\begin{definition}[Nature for MDPSTs] 
A \emph{nature} of an MDPST is a function
$\gamma_m: {\rm FPaths} \times A \to {\rm Dist}(\S)$ such that, $\gamma(\xi, a)\in \mathcal{H}_s^a$ for $\xi\in {\rm FPaths}$ and $a\in A({\ls}(\xi))$. 
\end{definition}

Suppose we fix a nature $\gamma_m$. The probability of an agent strategy $\sigma_m$ enforcing $\varphi$ is denoted by ${\rm Pr}_{\M_N}^{\sigma_{m}, \gamma_m}(\varphi) := {\rm Pr}_{\M_N}(\{\xi\in \Xi^{\sigma_m, \gamma_m} \mid \pi(\xi^k, \M_N) \models \varphi \text{ for some } k \geq 0\})$,
where $\Xi^{\sigma_m, \gamma_m}$ is the set of all probable paths generated by the agent strategy $\sigma_m$ and nature $\gamma_m$. 
We now define (optimal) robust strategies for MDPSTs, which quantify all natures.  

\begin{definition}[Robust strategy]\label{RobustSat}
    Let ${\mathcal{M}}_N$ be an MDPST, $\varphi$ an \LTLf formula, and $\beta \in [0,1]$ a threshold. An agent strategy $\sigma_m$ \emph{robustly enforces} $\varphi$ in ${\mathcal{M}}_N$ wrt $\beta$, if for every nature $\gamma_m$, the probability of the probably generated paths satisfying $\varphi$ is no less than $\beta$, that is, $P_{{\mathcal{M}}_N}^{\sigma_m} (\varphi)\ge \beta$, where $P_{{\mathcal{M}}_N}^{\sigma_m} (\varphi):={\min}_{\gamma_m}\{{\rm Pr}_{{\mathcal{M}}_N}^{\sigma_m, \gamma_m}(\varphi)\}.$
    Such $\sigma_m$ is referred to as a robust strategy for $\M_N$~(wrt $\beta$).
\end{definition}


\begin{definition}[Optimal robust strategy]
    An optimal strategy $\sigma^*_m$ robustly enforces an \LTLf formula $\varphi$ in an MDPST $\M_N$ is $\sigma_m^*=\arg{\max}_{\sigma_m}\{{\rm Pr}_{{\mathcal{M}}_N}^{\sigma_m} (\varphi)\}.$
     In this case, $\sigma^*_m$ is referred to as an \emph{optimal robust strategy} for $\M_N$.
\end{definition}




The problem of MDPSTs with \LTLf objectives is to compute an optimal robust strategy $\sigma^*_m$, which maximizes the probability of robustly enforcing $\varphi$ in $\M_N$, i.e., achieving the maximal value of $\beta$. Analogously, we can also define an MDPST with simple reachability, i.e., reaching a set of goal states. To avoid repetition, it has been omitted.



\subsection{Solution technique}\label{sec4-3}

The key idea to solve the TH problem for \LTLf planning in nondeterministic domains is to combine the quantifiable action-instruction errors and the unquantifiable adversarial nondeterminism of the domain into an MDPST. This MDPST has the same states as the domain $\N$, and incorporates the action-instruction errors and the domain's adversarial nondeterminism into its transitions. In this case, we reduce $\P^n$ to an MDPST with an \LTLf objective.

\noindent\textbf{Probabilistic abstraction.} 
We build an MDPST $\mathcal{M}_N = (\S, s_0, A, \F, \T_N, \L)$ from $\P^n = (\N, \varphi, \E)$ as follows. $S$, $s_0$, $A$, and $\L$ are the same as in $\N$. To construct the set-valued nondeterministic transition function $\F$ and the mass assignment function $\T_N$, we incorporate the action-instruction errors $\E$ into the probabilistic transitions, much like in deterministic domains. In nondeterministic domains, however, the successor for each state-action pair is not singular due to the adversarial environment behaviour. Consequently, the successor of each probabilistic transition is a set, with environment deciding which element of the set to transit to.

Based on these observations, we can construct the set-valued nondeterministic transition function $\F$ as $\F(s, a) = \cup_{a'\in A(s)} \{F_n(s, a')\}$, such that $\F(s, a)$ is a set of subsets in $\S$,
for $s\in \S$ and $a\in A(s)$. The corresponding mass assignment function $\mathcal{T}_N$ is such that 
    \begin{equation*}\label{ND_T1}
    \mathcal{T}_N(s, a, \Theta)= \begin{cases}
        err(s, a)(a'), & \text{if $a' \in A(s), \Theta=F_n(s, a'),$}\\
        0, & \text{otherwise.}            
    \end{cases}   
\end{equation*}

We now reduce $\P^n$ to an MDPST with an \LTLf objective. 

\begin{theorem}\label{thm:toMDPST}
    Let $\P^n = (\N, \varphi, \E)$ be a problem defined in Def.~\ref{ND}, and $\M_N$ the constructed MDPST described above. An optimal robust strategy for $\M_N$ with $\varphi$ is an optimal strategy for $\P^n$, and vice versa, that is: 
    \begin{equation*}
    \sigma^*_p=\arg\max_{\sigma_m}\{{\rm Pr}_{{\mathcal{M}}_N}^{\sigma_m} (\varphi)\},
    \end{equation*}
    where $\sigma^*_p$ is given in Def.~\ref{ND}, and ${\rm Pr}_{{\mathcal{M}}_N}^{\sigma_m} (\varphi)$ is in Def.~\ref{RobustSat}.
\end{theorem}
\begin{proof}
First, we observe that $\sigma_p = \sigma_m$, by construction. Second, one can derive, according to Definitions 4, 5, and 6 that 
$\cup_{\gamma_p}\Phi^{\sigma_p, \gamma_p, \E} = \cup_{\gamma_m}\Xi^{\sigma_{m}, \gamma_m}$ if $\sigma_p = \sigma_m$. 
The rest of the proof can be completed similarly to Theorem 1. 
\end{proof}
Finally, we propose an algorithm to solve the problem of MDPSTs with \LTLf objectives. This algorithm is based on a reduction to an MDPST with simple reachability, i.e., reaching a set of goal states. The algorithm is motivated by existing value iteration algorithms for MDPs with \LTLf objectives~\cite{Wells20}. Essential adaptations, however, are needed to handle the set-valued transitions in MDPSTs.

Given an MDPST $\M_N = (\S, s_0, A, \F, \T, \L)$ and an \LTLf formula $\varphi$, we first construct the corresponding DFA $Aut_\varphi = (2^{Prop}, \Q, q_0, \delta, acc)$ of $\varphi$. Then, the product MDPST $\M_N^\times = (
 \S^\times, s^\times_{0}, A^{\times}, \mathcal{F}^{\times}, \mathcal{T}^{\times}, \L^{\times})$ is constructed accordingly with  $\S^\times=\S\times Q, A^{\times}=A$.
The set of goal states is given by ${acc}^\times = \{(s, q)\in \S \mid q\in acc\}$. 

For the efficiency of strategy synthesis, we further introduce 
an optimization that computes an optimal robust strategy on a sub-MDPST, which only consists of states that are~(forward) reachable from the initial state $s^\times_0$ and~(backward) reachable from the set of accepting states ${acc}^\times$. To do so, we first partition  $\S^\times$ wrt the initial state $s^\times_0$ and the set of accepting states ${acc}^\times$. Specifically, 
let $S_r\subseteq \S^\times$ be the set of states that can be reached from $s^\times_0$. We now partition $\S^\times$ into $\S^\times=S_n\cup S_d\cup S_p$, where $S_n=\S^\times\setminus S_r$ consists of states that cannot be reached from $s_0$, $S_d\subseteq S_r$ consists of states reachable from $s_0$ but that cannot reach any states in ${acc}^\times$, and $S_p=S_r\setminus S_d$ includes those that can be reached from both the initial and accepting states. 

We construct a sub-MDPST $\mathcal{Z} = (\S_p, s_0, A_p, \F_p, \mathcal{T}_p, \L_p)$ from $\M^\times_N$ with respect to $\S_p$ as follows. $A_p = A^\times \cup \{a_\epsilon\}$, where $a_\epsilon$ denotes self-loop action. The set-valued transition function is such that $\F_p(s, a) = \F^\times(s, a), \forall s\in S_p\setminus {acc}^\times$ and $\F_p(s, a_\epsilon) = s, \forall s\in {acc}^\times$. 
The mass assignment function $\mathcal{T}_p$ is then given by i) $\mathcal{T}_p(s, a, \Theta)=\mathcal{T}^{\times}(s, a, \Theta), \forall \Theta\in \F_p(s, a)$ if $s\in S_p\setminus {acc}^\times, a\in A^\times$, and ii) $\mathcal{T}_p(s, a_\epsilon, s)=1$ if $s\in {acc}^\times$. 

Define a value function $V_{\mathcal{Z}}: \S_p \to \mathbb{R}_{\ge 0}$ by $V_{\mathcal{Z}}(s)
    =\max_{\sigma_m}\min_{\gamma_m} \{{\rm Pr}_{\mathcal{Z}}^{\sigma_m, \gamma_m}({acc}^\times)\},$ 
which represents the maximal probability of reaching ${acc}^\times$ from $s$. Then one can get that $V_{\mathcal{Z}}(s)=1, \forall s\in {acc}^\times$. For $s\in S_p\setminus {acc}^\times$, the Bellman principle of optimality is \cite{satia1973markovian}:
\begin{equation}\label{VF}
    V_{\mathcal{Z}}(s)
    =\max_{a\in A_p(s)}\min_{{\rm Pr}(\cdot| s, a)\in \H_s^a} \Big\{\sum_{s'\in \S_p} {\rm Pr}(s'| s, a)V_{\mathcal{Z}}(s')\Big\}.   
\end{equation}

Moreover, it was further shown in \cite{trevizan2007planning} that a simplified Bellman equation exists for MDPSTs. That is, one can safely pull the $\min$ operator inside the summation, which gives a more efficient variant
\begin{equation}\label{VF_deterministic}
\begin{aligned}
    V_{\mathcal{Z}}(s)
    =\max_{a\in A_p(s)}\Big\{\sum_{\Theta\in \mathcal{F}_p(s, a)}\mathcal{T}_p(s, a, \Theta)
    \min_{s'\in \Theta}\{V_{\mathcal{Z}}(s')\}\Big\}.
    \end{aligned}
\end{equation}

An optimal robust strategy $\sigma_m^*$ can be derived from $V_{\mathcal{Z}}$. Analogously to Sec.~\ref{sol:THDet}, due to cross product, every strategy $\sigma_m$ for $\M_N^\times$ wrt $acc^\times$ induces a strategy for $\P^n = (\N, \varphi, \E)$. Together with Thm.~\ref{thm:toMDPST}, we have the following.

\begin{theorem}\label{thm:toMDPSTfinal}
Let $\P^n = (\N, \varphi, \E)$ be as defined in Def. \ref{ND}, $\M_N^\times = \M_N \times Aut_{\varphi}$ the constructed product MDPST, and $acc^\times$ the set of goal states. The computed optimal robust strategy $\sigma^*_m$ of $\M_N^\times$ with reachability goal $acc^\times$ induces an optimal strategy $\sigma^*_p$ for $\P^n$.
\end{theorem}

\begin{remark}
Let $\bar{\bm{F}} = \max_{s\in S_p}\{\max_{a\in A_p(s)}|\F_p(s, a)|\}$ be an upper bound of $|\F_p(s, a)|$ for all $(s, a) \in S_p\times A_p$. Let $\epsilon$ be the convergence precision of the value iteration (i.e., Eqn. (\ref{VF_deterministic})).
The complexity for achieving an $\epsilon$-suboptimal solution is $\mathcal{O}(|S_p|^2|A_p|\bar{\bm{F}}\log \frac{1}{\epsilon})$.
We highlight that MDPSTs are a subclass of MDPIPs, and they admit more efficient strategy synthesis algorithms than general MDPIPs. This is because for general MDPIPs, each round of value iteration of its corresponding Bellman equation (Eqn. (\ref{VF})) involves two stages: 1) the inner minimization problem
is solved (e.g., using a bisection algorithm) for each state and each action
and 2) the value function is updated with  dynamic programming. 
\end{remark}

\section{Implementation and Experimental Results}

\begin{figure*}[th]
     \centering
    \includegraphics[width=.8\linewidth]{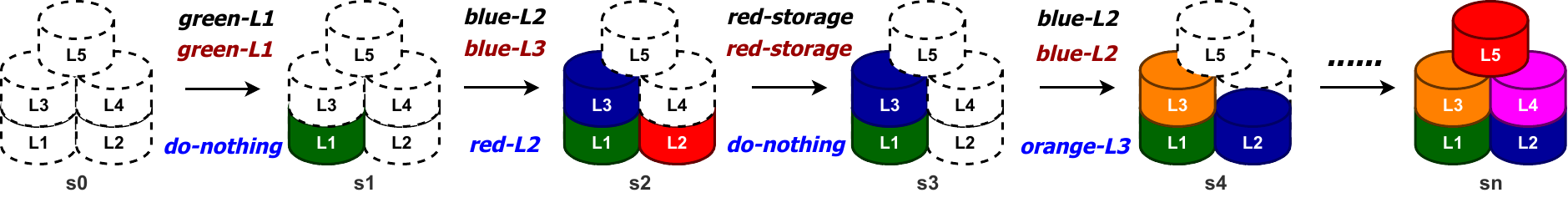}
         \caption{An execution example of an optimal strategy for the arch-building task. Robot-intended actions, robot-executed actions, and human interventions are shown in black, brick, and blue, respectively.}
    \label{fig:validation}
\end{figure*}

We implemented the solution technique described in Sec.~4, which subsumes the method described in Sect.~3,
in Python, and use LYDIA~\cite{de2021lydia} for \LTLf-to-DFA construction. The implementation details of our algorithms and experiments can be found on GitHub: \href{https://github.com/piany/Tremblinghand_LTLf}{https://github.com/piany/Tremblinghand\_LTLf}.

In this section, we present a case study to demonstrate the effectiveness of the proposed method. 
Our case study is based on the human-robot co-assembly problem, described in Example \ref{example2}. The ``trembling-hand" robot aims to stack blocks to a certain configuration with unpredictable human interventions. Note that we assume the human only has a limited number $K$ of moves (otherwise, the robot has no way to guarantee task completion \cite{he2019efficient}.). In particular, we consider the configurations of having certain objects in certain locations. The goal configurations for $N$ blocks are depicted in Figure~\ref{Fig:arch}, involving objects ranging from 2 to 6~(i.e., $|OBJ|\in \{2, 3, \cdots, 6\}$). 


\noindent\textbf{Implementation.}
The key challenge in the implementation is an effective representation of the planning domain with error actions. On one hand, it impacts the efficiency for model construction.
As shown in~\cite{WellsKLKV21}, the state space of the co-assembly problem grows exponentially in the number of objects. Therefore, it is extremely challenging to build a tractable model for a large number of objects. On the other hand, the efficiency of strategy synthesis is also impacted, as discussed in Remark 1.


In~\cite{WellsKLKV21}, three different choices are examined to encode the state space, integer encoding~(states are enumerated by breadth-first search), object encoding~(use tuples, e.g., $(1, 2, 0)$, mapping each object to its location), and location encoding~(using tuples, e.g., $(0, 1, 1, 0)$, mapping each location to the number of objects therein), where location encoding shows the best overall performance. 
We follow this observation and use location encoding for state-space representation. Note, however, that in our case a single tuple is not enough to represent a state. This is due to the one-to-one correspondence between objects and locations in the goal configuration. For example, in the configuration involving 5 objects~(where the goal configuration is shown on the upper right of Figure \ref{Fig:arch}), if the locations of the green and blue blocks are interchanged, the task is deemed incomplete. 
To resolve this issue, we use a tuple of tuples for state encoding. For instance, a state $((0, 0, 0, 1), (1, 0, 0, 0), (0, 1, 0, 0))$ means that $Obj_1$ is at location 3, $Obj_2$ is at the storage~(the first element in each tuple represents storage), and $Obj_3$ is at location 1. Therefore, given $N$ objects, the co-assembly domain has at most $2^{N(N+1)}$ states. Luckily, many of the states are invalid due to physical constraints, e.g., one block cannot be at two different locations, and each location (except for storage) can have at most one block. Using these physical constraints, we design a recursive algorithm (to avoid enumerating all possible states, hence greatly improving efficiency) to prune the co-assembly state space such that only valid states are maintained. The details of the state pruning algorithm can be found in Appendix. 

Moreover, the $K$-move limitation of the human should also be considered. We use a counter $\C = \{0, 1, \cdots, K\}$ to record the human moves. Together with the (pruned) state-space $X$ of domain $\N = (\S, s_0, A, F_n, \L)$, i.e., $X \subseteq \S$, we obtain an augmented state-space $\S' = X \times \C$. 
If the counter value is $K$, then the set of applicable human actions is restricted to ${\textsc{Do-nothing}}$ for all $s \in \S'$. Finally, we can construct the probabilistic abstraction of the co-assembly problem in the form of an MDPST over $\S'$ following Sec.~\ref{sec4-3}.



Figure~\ref{fig:validation} depicts a~(simplified) execution example of an optimal strategy for an arch-building task with 5 blocks at 5 certain locations~(see the right-most arch).
All blocks are in storage~(state $s0$) at the beginning. Then, the robot intends to put the green block at $L1$, which succeeds without any interference from the human~(state $s1$). Next, the robot intends to finish building the base level by putting the blue block at $L2$. However, due to the trembling hand, the blue block was put to $L3$. Furthermore, the human put the red block at $L2$ to prevent the robot from building the arch~(state $s2$). Note that the red block is supposed to be at $L5$~(see the right-most arch). In this case, the robot intends to remove the red block, which succeeds~(state $s3$). After a finite number of executions, in spite of the trembling hand and the interventions from the human, the robot builds the arch~(state $sn$).

\noindent\textbf{Experimental results.}
In our experiments, the convergence precision for the value iteration in Eqn.~(\ref{VF_deterministic}) was set to $10^{-3}$. All experiments were carried out on a Macbook Pro (2.6 GHz 6-Core Intel Core i7 and 16 GB of RAM). 

We first show the effectiveness of the state-pruning technique. For the case that $K = 3$, the number of states and transitions in the constructed MDPST are shown in Figure \ref{Fig:state_transition} for different numbers of objects ($2\le |OBJ|\le 6$). It is worth noting that six is the maximum number of objects that have been considered in the literature~\cite{WellsKLKV21} (due to the exponential blowup in the number of states). It is evident that the state space, post-pruning, exhibits much slower growth compared to exponential expansion.


\vspace{-3mm}
    \begin{figure}[H]
\centering
\includegraphics[width=0.36\textwidth]{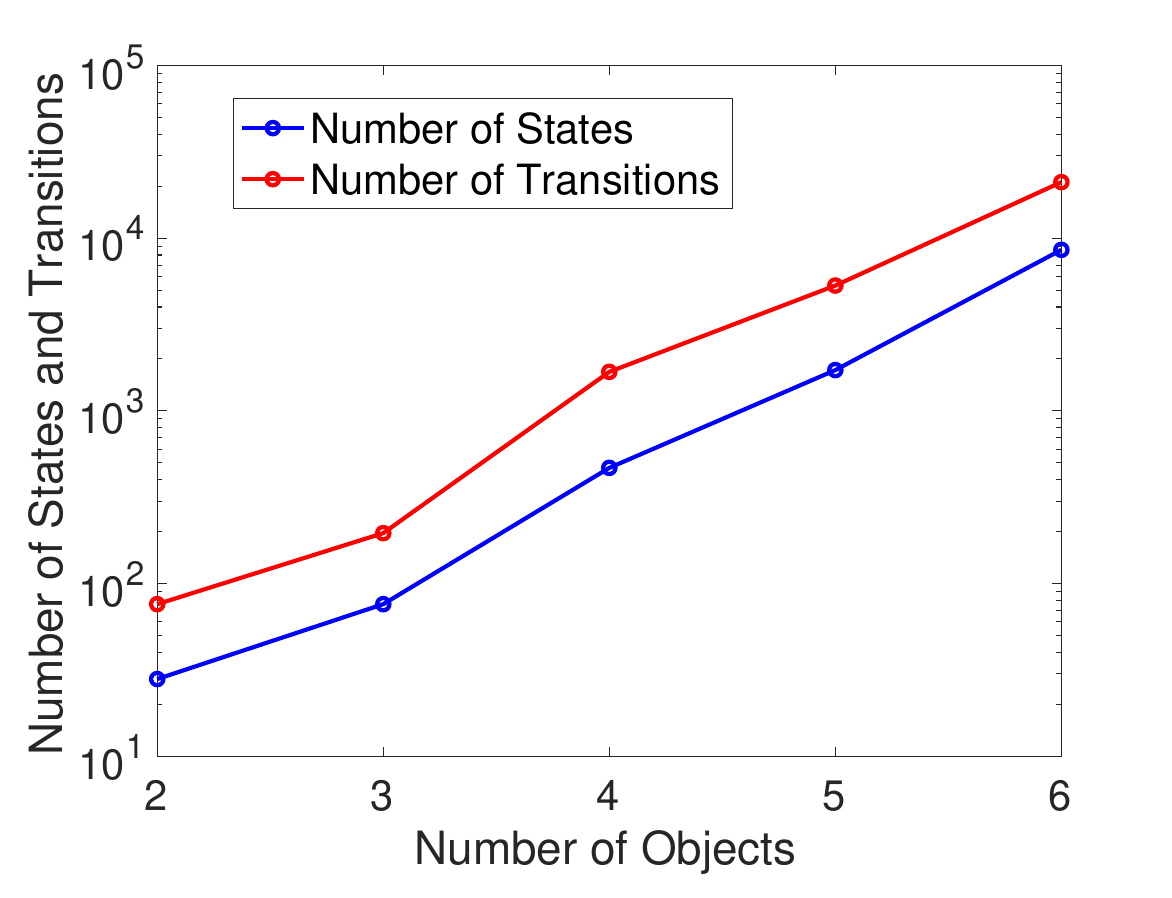}
	\caption{Number of states and transitions in the constructed MDPST for $2\le |OBJ|\le 6$ (in log scale). }
	\label{Fig:state_transition}
\end{figure}

\vspace{-8mm}
    \begin{figure}[H]
\centering
\includegraphics[width=0.36\textwidth]{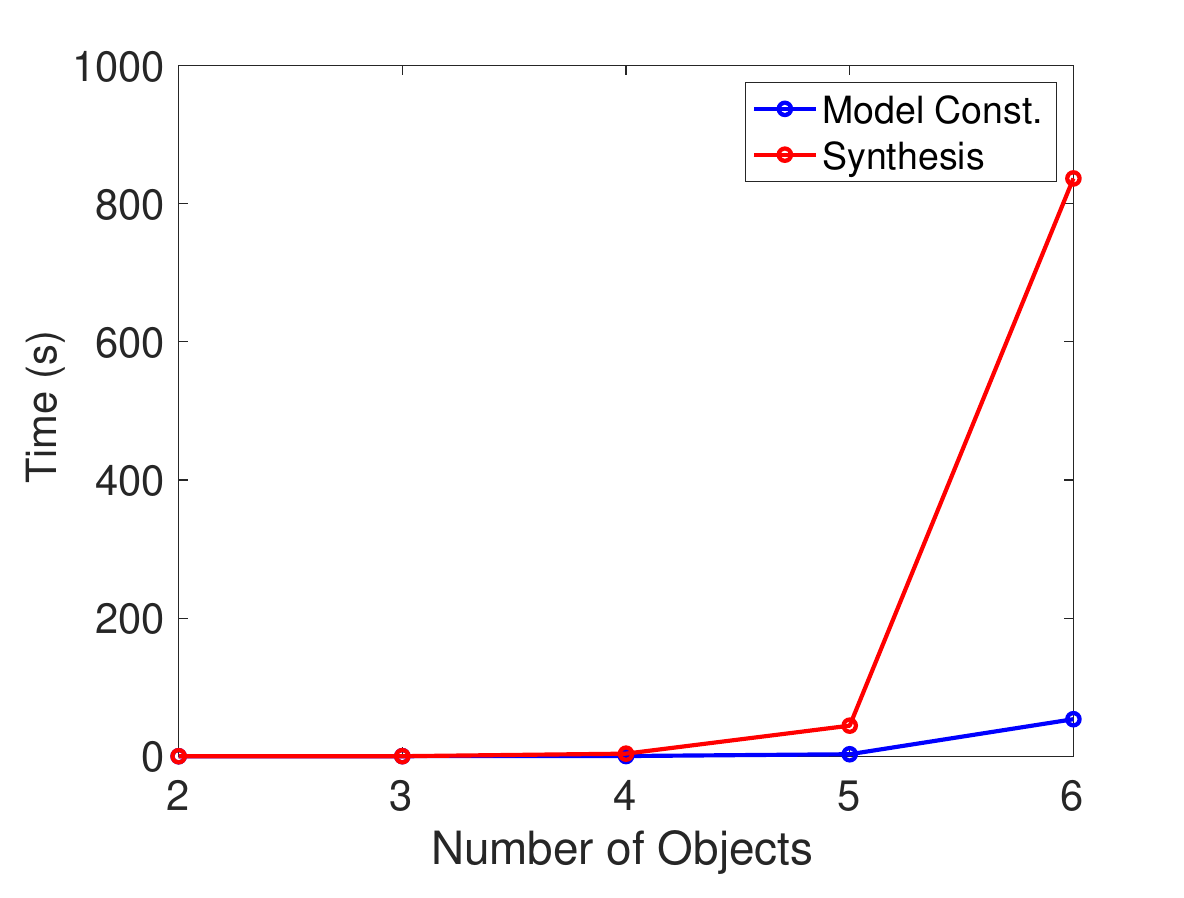}
	\caption{Computation time for model construction and strategy synthesis for $2\le |OBJ|\le 6$. }
	\label{Fig:computation_time}
\end{figure}

In addition, the computation time for model construction (including state-space pruning and MDPST construction) and strategy synthesis (including DFA construction, product MDPST computation, and robust value iteration of Eqn. (\ref{VF_deterministic})) with respect to different number of objects ($2\le |OBJ|\le 6$) are depicted in Figure \ref{Fig:computation_time}. Note that \LTLf synthesis in nondeterministic domains is 2EXPTIME-complete in the size of the \LTLf formula and EXPTIME-complete in the size of the domain \cite{de2023ltlf}. Moreover, the probabilistic behaviour of the agent~(caused by the ``trembling hand") further complicates the synthesis problem.
Nevertheless, one can see that, for the case of six objects, the overall problem (including model construction and strategy synthesis) can still be solved quite efficiently using the proposed algorithm.

For the case of 5 objects, we further explore situations where the upper bounds on human interventions range from 3 to 8 (i.e., $K\in \{3, 4, \cdots, 8\}$). 
The experimental results are shown in Table 1, where the size of the problem, i.e., the number of states and transitions, and the computation
time for MDPST construction and strategy synthesis with respect to $K$, are listed. It is clear that, although both the number of states and transitions grow linearly as $K$ increases~(as does the model construction time), the synthesis time grows faster. This is because computing an $\epsilon$-suboptimal solution for an MDPST $\M_N = (
 \S, s_{0}, A, \mathcal{F}, \mathcal{T}_N, \L)$, which consists of reachable states only~(following the partition optimization described in Section~4.3), is $\mathcal{O}(|\S|^2|A|\bar{\bm{F}}\log \frac{1}{\epsilon})$.

\begin{table}[t]
	\centering
	\caption{Model construction and synthesis computation times for
5 objects.}  	\label{Table_comtime} 
\begin{tabular}{l|l|l|l|l}
\hline
\multicolumn{1}{c|}{$K$} & \multicolumn{1}{c|}{States} & \multicolumn{1}{c|}{Transitions} & \multicolumn{1}{c|}{\begin{tabular}[c]{@{}c@{}}Model Const.\\ (s)\end{tabular}} & \multicolumn{1}{c}{\begin{tabular}[c]{@{}c@{}}Synthesis\\ (s)\end{tabular}} \\ \hline\hline
3 &1724 &5324 &3.222 &45.81 \\ \hline
4 &2155 &6655 &4.685 &78.820 \\ \hline
5 &2586 &7986 &6.958 &132.42 \\ \hline
6 &3017 &9317 &9.283 &190.203 \\ \hline
7 &3448 &10648 &12.383 &263.826 \\ \hline
8 &3879 &11979 &15.768 &347.015 \\ \hline
\end{tabular}
\end{table}
\section{Conclusions}

In this paper, we have investigated the trembling hand problem for \LTLf planning in deterministic and nondeterministic domains. We formulate the problem formally by defining action-instruction errors and perturbed paths influenced by these errors. For the case of deterministic domains, we show that the problem can be reformulated as an MDP with an \LTLf objective, leveraging existing algorithms for synthesis. In the case of nondeterministic domains, on the other hand, we utilise 
MDPSTs with \LTLf objectives, for which we propose an efficient robust value iteration algorithm for synthesis. In particular, MDPSTs with \LTLf objectives have been studied here for the first time. 
We also demonstrate the promising scalability of the proposed algorithm in a case study. For future work, we plan to leverage symbolic techniques for synthesis, aiming to improve efficiency.

\section*{Appendix: State Pruning Algorithm}


\begin{algorithm}[t]
\caption{\textit{State Pruning}}
\caption{\textit{State Pruning}}
\begin{algorithmic}[1]
\Require the number of objects $N$.
\Ensure the set of valid states $S_{valid}(N)$.
\For {$i = 2, \cdots, N$}
\If {$i =2$}
\State $\mathbb{S} \leftarrow \textit{CreateStates}(i, i+1)$
\For {$s \in \mathbb{S}$}
\If {$\textit{RowsSumIsOne}(s)$ and $\textit{ColsSumIsOne}(s)$}
\State add $s$ into $S_{valid}(i)$
\EndIf
\EndFor
\EndIf
\If {$i \ge 3$}
\State $r_{valid} \leftarrow \textit{CreateRow}(i)$
\For {$s_{i-1} \in S_{valid}(i-1)$}
\For {$r(k)$ in $s_{i-1}$}
\State $r_{new}(k) \leftarrow \{r(k) + (0), (0, \cdots, 0, 1)\}$
\EndFor
\State $S_{comb} \leftarrow (r_{new}(0), \cdots, r_{new}(i-2), r_{valid})$
\For {$s_{expand}$ in $\textit{Product}(S_{comb})$}
\If {$\textit{RowsSumIsOne}(s_{expand})$ and $\textit{ColsSumIsOne}(s_{expand})$}
\State add $s_{expand}$ to $S_{valid}(i)$
\EndIf
\EndFor
\EndFor
\EndIf
\EndFor
\end{algorithmic}
\end{algorithm}

Given $N$ objects, the co-assembly domain has at most $2^{N(N+1)}$ states. To obtain an effective representation of the co-assembly domain, we design a state pruning algorithm~(Algorithm 1), which takes into account the physical constraints: (1) each block cannot be at more than one location, and (2) each location (except for storage) can have at most one block. In this way, only a set of valid states that satisfy these constraints are maintained for strategy synthesis. Note that Algorithm 1 employs a recursive approach, strategically avoiding the enumeration of all possible states. This recursive nature enhances efficiency.

Algorithm 1 takes as input the number of objects $N$ and outputs the set of valid states $S_{valid}(N)$. It starts by computing $S_{valid}(2)$ for 2 objects (lines 2-9),  followed by the recursive computation of $S_{valid}(\geq 3)$ through $S_{valid}(N)$ (lines 10-23). When $i =2$, we first create all possible states $\mathbb{S}$ using $\textit{CreateStates}(i, i+1)$, where each state is a $2\times 3$ tuple and each element of the tuple is either 0 or 1 (line 3). Then for a state $s\in \mathbb{S}$, if all the rows sum equal to 1~(checked by $\textit{RowsSumIsOne}$) and all the columns~(except for column 0, which represents storage) sum equal to 1 (checked by $\textit{ColsSumIsOne}$), then $s$ is added to $S_{valid}(2)$ (lines 4-8). 

When $i\ge 3$, we compute $S_{valid}(i)$ by expanding each state $s_{i-1}\in S_{valid}(i-1)$ (which is a $(i-1) \times i$ tuple) into a set of valid states $s_i$ (which is a $i \times (i+1)$ tuple). In line 11, sub-algorithm $\textit{CreateRow(i)}$ returns a set of $1\times (i+1)$ tuples whose row sum equals to 1 (i.e., $r_{valid}$). In lines 13-15, we expand each row $r(k)$ of $s_{i-1}$ into a set of two rows, i.e., $r_{new}(k)$, by adding one element (either 0 or 1) to the end of $r(k)$. Then, in line 17, sub-algorithm $\textit{Product}(\cdot)$ computes a set of tuples containing all possible combinations of elements from $r_{new}(0), \cdots, r_{new}(i-2)$ and $r_{valid}$ (i.e., $S_{comb}$). For each $s_{expand}$ in $S_{comb}$, if all the rows sum equal to 1  and all the columns sum (except for column 0) equal to 1, the state $s_{expand}$ is added to $S_{valid}(i)$ (lines 17-21).

\section*{Acknowledgements}

This work was supported in part by the ERC ADG FUN2MODEL
(No. 834115), in part by the ERC ADG WhiteMech (No. 834228), in part by the EU ICT-48 2020 project TAILOR (No. 952215), in part by NSF grants IIS-1527668, CCF-1704883,
IIS-1830549, CNS-2016656, DoD MURI grant N00014-20-1-2787,
and in part by an award from the Maryland Procurement Office.

\bibliographystyle{named}
\bibliography{references,ref}

\end{document}